\documentclass[11pt]{article}
\usepackage{amssymb, latexsym, mathtools, amsthm}
\usepackage{authblk}
\usepackage{stmaryrd}
\usepackage{multirow}
\makeatletter

\let\c@table\c@figure
\makeatother 

\renewenvironment{abstract}
 { \normalsize
  \list{}{\setlength{\leftmargin}{.0cm}%
    \setlength{\rightmargin}{\leftmargin}}%
  \item {\bf \abstractname.}\relax}
 {\endlist}

\begingroup
    \makeatletter
    \@for\theoremstyle:=definition,remark,plain\do{%
        \expandafter\g@addto@macro\csname th@\theoremstyle\endcsname{%
            \addtolength\thm@preskip\parskip
            }%
        }
\endgroup
\usepackage{pifont}
\usepackage{booktabs}
\usepackage{siunitx}
\newcommand{\cmark}{\ding{51}}
\newcommand{\xmark}{\ding{55}}

\usepackage{enumerate}
\usepackage{pgf,tikz}\usetikzlibrary{decorations, decorations.pathmorphing}\usetikzlibrary {shapes}
\usepackage{pdfsync}
\usepackage[numbers]{natbib}
 \usepackage{graphicx}
\usepackage{parskip}
\usepackage{tcolorbox}
\usepackage{tabularx, colortbl}
\usepackage{geometry}
\theoremstyle{plain}
\newtheorem{thm}{Theorem}[section]
\newtheorem*{thmx}{Theorem}
\newtheorem{prop}[thm]{Proposition}

\newtheorem{lem}[thm]{Lemma}

\theoremstyle{definition}

\newtheorem{defi}[thm]{Definition}

\newcommand{\Nat}{\mathbb{N}}

\newcommand{\restr}{\hspace{-0.05cm}\upharpoonright\hspace{-0.05cm}}  

\DeclarePairedDelimiter{\floor}{\lfloor}{\rfloor}

\newcommand{\bigon}{\textrm{\textup O}\hspace{0.02cm}(1)}

\newcommand{\absb}[1]{\big|\hspace{0.03cm}{#1}\hspace{0.03cm}\big|}

\newcommand{\abs}[1]{|\hspace{0.05cm}{#1}\hspace{0.05cm}|}

\newcommand{\parb}[1]{\big({#1}\big)}

\newcommand{\sqbrad}[2]{\left\{\hspace{0.05cm}{#1}\ :\ {#2}\hspace{0.05cm}\right\}}

\newcommand{\ce}{c.e.\ }

\newcommand{\pf}{prefix-free }
\newcommand{\fand}{\ \wedge\ }

\newcommand{\XX}{\textsf{\textup X}}
\newcommand{\OO}{\textsf{\textup O}}
\newcommand{\TT}{\textsf{\textup T}}

\newcommand{\XXss}{\hspace{0.05cm}\textsf{\textup X}}
\newcommand{\OOss}{\hspace{0.05cm}\textsf{\textup O}}
\newcommand{\TTss}{\hspace{0.05cm}\textsf{\textup T}}

\newcommand{\twome}{2^{\omega}} 
\newcommand{\twomel}{2^{<\omega}}

\newcommand{\geqa}{\overset{\textrm{\tiny $+$}}{\geq}}
\newcommand{\eqa}{\overset{\textrm{\tiny $+$}}{=}}

\newcommand{\wedga}{\ \wedge\ }
\newcommand{\asto}{^{\ast}}
\newcommand{\hthree}{\hspace{0.3cm}}

\title{Compression of enumerations and gain\thanks{Supported by Beijing Natural Science Foundation (IS24013).}}
\author{George Barmpalias}\author{Xiaoyan Zhang} \author{Bohua Zhan\thanks{Authors are in alphabetical order..} }
\affil{State Key Lab of Computer Science, Institute of Software\\ Chinese Academy of Sciences, Beijing, China}

\begin{document}
\maketitle
\begin{abstract}
\noindent 
We study the compressibility of  enumerations in the context of 
Kolmogorov complexity, focusing on strong and weak forms of compression and their    
{\em gain}:  the amount of auxiliary information embedded in the compressed enumeration.
The existence of strong compression and weak gainless compression is shown for any computably enumerable (c.e.) set.
The density problem of c.e.\ sets with respect to their prefix complexity  is reduced to the 
question of whether every c.e.\ set  is well-compressible, which we study via enumeration games.
\end{abstract}

\section{Introduction}\label{VjB4aXSE24}
Given an effective enumeration of a set $A\subseteq\Nat$, we are interested in
obtaining a {\em compression} of it, in the form of an  enumeration of another set $D$ which:
\begin{equation}\label{HRsx2H3tm}
\textrm{essentially contains the information in $A$ in a  {\em compact} form.}
\end{equation}
In \S\ref{5OMy18z6Fq} we will see that this is the key to an open problem in Kolmogorov complexity of c.e.\ sets, 
but it is also  interesting in its own right.

To be specific we express \eqref{HRsx2H3tm} in terms of Kolmogorov complexity. 
We identify sets $A\subseteq\Nat$ with their characteristic sequence, where a 1 on position $i$ 
indicates the membership of $i$ in $A$. In this way $A\restr_{n}$ denotes the $n$-bit prefix of $A$  
but also the restriction of the set $A$ to members $<n$.

By {\em essentially} in \eqref{HRsx2H3tm} we mean indifference to finite errors: for certain non-decreasing sequence $(\ell_n)$ with $\ell_n\leq  n$, the strings $D\restr_{\ell_n}$ are effectively mapped to an $n$-bit string of constant Hamming-distance from $A\restr_n$.
When $A, D$ are computably enumerable
(c.e.), the latter can be expressed as
\[
C(A\restr_n\mid D\restr_{\ell_n})=\bigon
\]
where $C(\sigma\mid\tau)$ denotes the conditional Kolmogorov complexity:
the length of the shortest program that can generate $\sigma$ from input $\tau$.
Requiring that
\[
\ell_n\ll n
\hspace{0.3cm}\textrm{or}\hspace{0.3cm}
\parb{\ell_n=n\wedga \abs{D\restr_n}\ll \abs{A\restr_n}}
\]
where $\ll$ means `considerably smaller than', is a natural way to express that 
$D$ is more {\em compact} than $A$ and its enumeration is a compression of the one of $A$.
\begin{defi}[Strong compression]\label{mnfu42bwoIa}
Given \ce sets $A, D\subseteq\Nat$ such that
\begin{equation}\label{WK2nBAOK8O}
C(A\restr_n \mid D\restr_{\floor{n/2}})=\bigon
\end{equation}
we say that $D$ is  a {\em strong compression of $A$} with {\em gain} $n\mapsto C(D\restr_{\floor{n/2}} \mid A\restr_n)$. 
If  
$$C(D\restr_{\floor{n/2}} \mid A\restr_n)=\bigon$$
then $D$ is  a {\em gainless}  strong compression of $A$, and $A$ is  
{\em well-compressible}.
\end{defi}
Condition \eqref{WK2nBAOK8O} indicates that the information in $A\restr_n$ is contained in the
first $\floor{n/2}$ bits of $D$. This is a strong form of compression which, as we will soon see, often forces additional information 
into $D\restr_{\floor{n/2}}$: information that is not recoverable from $A\restr_n$. The latter
is quantified by the {\em gain} of the strong compression.

At this point the reader may be questioning the choice of $\ell_n:=\floor{n/2}$ in Definition
\ref{mnfu42bwoIa}.  Kolmogorov complexity, however, is only precise up to a constant, 
so the notion of {\em strong $\epsilon$-compressibility}  we get by replacing  $1/2$ by any $\epsilon\in (0,1)$ 
is equivalent, in the sense that given any \ce set $A$, by iteration: 
\[
\textrm{if $A$  has  a strong compression,  it also has a strong  $\epsilon$-compression. }
\]
By \citep[Corollary 2.6 and Theorem 2.7]{iplBarmpaliasHLM13}   
the halting problem $H$ with respect to the standard  numbering of all programs (known as a Kolmogorov numbering) 
is well-compressible (it has a gainless strong compression).\footnote{It is shown that $C(H\restr_n\mid H\restr_{\floor{n/2}})=\bigon$ 
while $C(H\restr_{\floor{n/2}}\mid H\restr_n)=\bigon$ is immediate.} By the same result:
\[
\textrm{every linearly-complete \ce set  is well-compressible.}
\]
It is not known if every \ce set is well-compressible. 

In \S\ref{4r6M2bsUrA} we give a method for strongly compressing any given \ce set.
%
\begin{thm}\label{nolyEgDBIB}
Given any computable enumeration of any set we can effectively enumerate a strong compression of it.
\end{thm}
Despite the simplicity of this method, there is a drawback: it often gives compressions with non-trivial gain; 
it is not gainless. As hinted above and shown in \S\ref{5OMy18z6Fq}, gainless compression is key to 
 the density problem for Kolmogorov complexity of \ce sets.
This is hard to achieve in general, so we  introduce a weaker notion.

\begin{defi}[Compression]\label{mnfu42bwoI}
Given \ce sets $A, D\subseteq\Nat$ such that 
\begin{equation}\label{rYL1DzSQhw}
\abs{D\restr_n}\leq \abs{A\restr_n}/2
\hspace{0.5cm}\textrm{and}\hspace{0.5cm}
C(A\restr_n \mid D\restr_n)=\bigon
\end{equation}
we say that $D$ is  a {\em compression of $A$} with {\em gain} $C(D\restr_n \mid A\restr_n)$. 

If $C(D\restr_n \mid A\restr_n)=\bigon$ we say that $D$ is  a {\em gainless} compression of $A$.
\end{defi}
The  {\em gain} in Definitions \ref{mnfu42bwoI} and \ref{mnfu42bwoIa} has distinct formulations that correspond to the 
underlying conditions \eqref{rYL1DzSQhw} and \eqref{WK2nBAOK8O} of the two types of compression.

Our main result is:
\begin{thm}\label{VXOXdPQyhO}
Given any computable enumeration of a set $A$ we can effectively enumerate a gainless compression $D\subseteq A$ of it. 
\end{thm}
The study of compressibility of enumerations was motivated by
a challenge \cite[Question 10]{BarmpBSLHLM13} in the the study of Kolmogorov complexity: 
\begin{equation}\label{HgFLf2jT2R}
\textrm{are the \ce sets dense with respect to  initial segment complexity?}
\end{equation}
\begin{table}
\colorbox{black!3}{\arrayrulecolor{white!20!black} 
\begin{tabular}{llcll}
{\small\em Preorder} &\hspace{0.2cm} &{\small\em } &\hspace{0.2cm}  &  {\small\em Condition}  \\[0.1ex]\cmidrule[0.5pt]{1-5}
{\small Relative Kolmogorov}  &\hspace{0.2cm} &{\small $\leq_{rK}$} & \hspace{0.2cm} & {\small $C(x\restr_n\mid y\restr_n)=\textrm{O(1)}$}\\[1.2ex] 
{\small Plain Kolmogorov}  &\hspace{0.2cm} &{\small $\leq_C$ } & \hspace{0.2cm} & {\small $C(x\restr_n)\leq C(y\restr_n)+\textrm{O(1)}$}\\[1.2ex] 
{\small Prefix-free Kolmogorov}  &\hspace{0.2cm} &{\small $\leq_K$ } & \hspace{0.2cm} & {\small $K(x\restr_n)\leq K(y\restr_n)+\textrm{O(1)}$}\\[1.2ex] 
\end{tabular}}\centering
\caption{Measures of relative initial segment complexity}\label{dsABuEidRZu}
\end{table}
To be precise, sets can be classified according to their Kolmogorov complexity via the
preorders introduced by \citet{MR2030512} and shown in Table \ref{dsABuEidRZu}, 
where $K$ denotes the Kolmogorov complexity with respect to \pf machines.
 
Given $r\in\{rK, C, K\}$ we can state question \eqref{HgFLf2jT2R} formally:
\begin{equation}\label{vNwabMEGSD}
\textrm{given  \ce sets $B<_r A$ is there a \ce set $F$ with $B<_r F <_r A$?}
\end{equation}
where $B<_r A$ denotes that $B\leq_r A$ and $A\not\leq_r B$. This question remains
open despite the substantial work in  \citep{iplBarmpaliasHLM13, BLibT, apal/Day10} that we review in \S\ref{ZUXHugFEb}. 

Well-compressibility (Definition \ref{mnfu42bwoIa}) 
 is essential in combining two enumerations $A, B$ into one, containing  {\em precisely} the combined
information of $A, B$ in its corresponding prefixes, and {\em no additional information}.
This is crucial in interpolation (density) constructions in computability and its absence  can be used 
to obtain non-density with respect to related preorders \citep{BLibT, apal/Day10}. 

In \S\ref{5OMy18z6Fq} we demonstrate that  well-compressibility is key to answering \eqref{vNwabMEGSD}:
\begin{thm}\label{ba9Ccy8LTPa}
Let $A, B$ be \ce sets and $r\in \{rK,K,C\}$. If $A, B$ are well-compressible 
and $B<_{r} A$, there is a  \ce set $F$ with $B<_{r} F<_{r} A$.
\end{thm}
This reduces \eqref{vNwabMEGSD} to the following question:
\begin{equation}\label{smQiu9zIIV}
\textrm{is every \ce set  well-compressible?}
\end{equation}
In \S\ref{pdP8xPFfiU} we 
reduce \eqref{smQiu9zIIV} to the solution 
of a two-player game, the {\em balance game}, namely
finding a winning strategy for one of the players. Despite the apparent simplicity of the 
balance game, finding a winning strategy is challenging. For this reason we study
a simpler version,  the {\em $k$-even game}, which we solve for $k<4$. 

We conclude in \S\ref{dzforlz6pr} with a summary and open problems.

\section{Complexity of enumerations}
Compressibility of effective enumerations was first studied by \citet{BarzdinsCe} in the
early years of Kolmogorov complexity. 
In \S\ref{ZUXHugFEb} we review the
state-of-the-art on this topic, focusing on aspects that are  related to our contribution.
In \S\ref{4r6M2bsUrA} we obtain a strong compression of any given \ce set, thus establishing
Theorem \ref{nolyEgDBIB}.

For simplicity, in this section we adopt the convention that 
\[
\textrm{\em all inequalities involving $C$ or $K$ are up to a constant.}
\]
For example $\forall n\ C(A\restr_n \mid n)\leq \log n$  means that  $\exists c\ \forall n,\ C(A\restr_n\ \mid n)\leq \log n + c$.

Similarly $\exists^{\infty} n\ C(A\restr_n)\geq 2\log n$  means  that $\exists c\ \forall k\ \exists n>k: C(A\restr_n)\geq 2\log n-c$.

By  {\em Kolmogorov complexity of a \ce set $A$} we mean the the complexity of its prefixes, namely 
the maps $n\mapsto C(A\restr_n)$ and $n\mapsto K(A\restr_n)$. 

\subsection{Kolmogorov complexity of \ce sets}\label{ZUXHugFEb}
\citet{BarzdinsCe}  showed that the  Kolmogorov complexity of any  \ce set $A$ satisfies
\[
\forall n\ C(A\restr_n\ \mid n)\leq \log n\wedga
\forall n\ C(A\restr_n)\leq 2\log n
\]
and constructed a \ce set $B$ with $\forall n\ C(B\restr_n)\geq \log n$. 
By \citet{holzlMerkKra09} 
\[
\exists^{\infty} n\ \parb{C(A\restr_n\mid n)\ =\bigon\wedga C(A\restr_n)\leq \log n}
\]
for every \ce set $A$. \citet{Kummer96Cce} constructed a \ce set $B$ with 
\[
\exists^{\infty} n,\ C(B\restr_n)\geq 2\log n.
\]
These demonstrate the optimality of Barzdins' logarithmic bounds. 
Similar results hold for the prefix-free Kolmogorov complexity $K$ of \ce sets \cite{koba_rod, holzlMerkKra09}.

Despite its narrow range, the descriptive complexity of \ce sets  $A$ has a rich structure
(see \cite{SolovProcEntro, Kummer96Cce, iplBarmpaliasHLM13, unilow} and \cite[\S 3]{BarmpBSLHLM13}) 
with the extremes characterized by:
\begin{itemize}
\item $A$ is computable iff $C(A\restr_n)\leq \log n$ (\citet{Chaitin:76})
\item $A$ is linearly-complete iff $C(A\restr_n) \geq \log n$  (\citet{iplBarmpaliasHLM13}).
\end{itemize}
Additional facts  can be found in \citep[\S 16]{rodenisbook} and \citep{MR1872278, unilow, apalBarm13enumK}.

\begin{table}
\colorbox{black!3}{\arrayrulecolor{white!20!black} 
\begin{tabular}{llcll}
{\small\em Preorder} &\hspace{0.2cm} &{\small\em } &\hspace{0.2cm}  &  {\small\em Condition}  \\[0.1ex]\cmidrule[0.5pt]{1-5}
{\small Identity bounded}  &\hspace{0.2cm} &{\small $\leq_{ibT}$ } &	\hspace{0.2cm}  & {\small oracle-use $n$ } \\[0.8ex]
{\small Computably Lipschitz} &\hspace{0.2cm}  &{\small $\leq_{cl}$ } & \hspace{0.2cm}  &  {\small  oracle-use $n+O(1)$}\\[0.8ex]
\end{tabular}}\centering
\caption{Strong variants of the Turing reducibility}\label{ABuEidRZu}
\end{table}

\subsection{Measures of relative randomness}
Relations $\leq_{rK}, \leq_{K}, \leq_{C}$ of  Table \ref{dsABuEidRZu} and their associated degree structures classify
the c.e.\ sets by their Kolmogorov complexity:

\begin{itemize}
\item the bottom degree in $\leq_{rK}, \leq_{C}$ consists of the computable sets. \citep{Chaitin:76}.
\item the bottom degree in $\leq_{K}$ includes noncomputable sets. \citep{Solovay:75, MRtrivrealsH}.
\item  $\leq_{rK}$ implies Turing reducibility. \citep{MR2030512} 
\item there is a maximum \ce degree with respect to $\leq_{rK}, \leq_{C}, \leq_{K}$. \citep{iplBarmpaliasHLM13}.
\end{itemize}
Similar classifications exist in terms of the preorders of Table \ref{ABuEidRZu},
introduced in \cite{MR2030512, SoareDiff2004} and 
 based on restricting the oracle-use 
of a Turing reduction to a slow-growing function.
These have been extensively studied but 
differ significantly from $\leq_{rK}, \leq_{K},\leq_{C}$ as they lack maximum and least upper bound in the c.e.\ sets.
The first part of Table \ref{ABuEidRZua} offers additional comparisons 
which are relevant to the present work. 

The density problem of $\leq_{rK}, \leq_{K},\leq_{C}$ in the \ce sets is discussed in \S\ref{5OMy18z6Fq}.

\begin{table}
\setlength\fboxsep{6pt}\colorbox{black!3}{\arrayrulecolor{white!20!black} 
\begin{tabular}{cccc}
{\em\small c.e.\ sets}   &  {\em\small dense} &  {\em\small  max} &  {\em\small join}  \\[0.1ex]\cmidrule[0.5pt]{1-4}
{\small $\leq_{ibT}$ }   & {\small \xmark} & {\small \xmark} & {\small \xmark} \\[0.1ex]
{\small $\leq_{cl}$}  &  {\small \xmark} & {\small \xmark} & {\small \xmark}\\[0.1ex]
{\small $\leq_{K}$}   & {\small\bf ?} & {\small \cmark} & {\small\bf ?} \\[0.1ex] 
{\small $\leq_{C}$}  & {\small\bf ?}& {\small \cmark} & {\small\bf ?} \\[0.1ex] 
{\small $\leq_{rK}$}   & {\small\bf ?} & {\small \cmark} & {\small\bf ?} \\[0.1ex] 
\end{tabular}}\hspace{0.6cm}
\setlength\fboxsep{4pt}\colorbox{black!3}{\arrayrulecolor{white!20!black} 
\begin{tabular}{cccc}
{\small\em  left-c.e.\ reals}   &  {\em\small dense} &  {\em\small  max} &  {\em\small join} \\[0.1ex]\cmidrule[0.5pt]{1-4}
{\small $\leq_{ibT}$}   & {\small\bf ?} & {\small \xmark} & {\small \xmark}\\[0.1ex]
{\small $\leq_{cl}$}  &  {\small\bf ?} & {\small \xmark}& {\small \xmark} \\[0.1ex]
{\small $\leq_{K}$}   & {\small \cmark}& {\small \cmark}& {\small \cmark} \\[0.1ex] 
{\small $\leq_{C}$}  & {\small \cmark}& {\small \cmark}& {\small \cmark}\\[0.1ex] 
{\small $\leq_{rK}$}   & {\small \cmark} & {\small \cmark}& {\small \cmark}\\[0.1ex] 
\end{tabular}
}
\centering
\caption{Structural properties of c.e.\ sets and left-c.e.\ reals}\label{ABuEidRZua}
\end{table}

\subsection{Strong compression of enumerations}\label{4r6M2bsUrA}
We show how to obtain a strong compression for any given \ce set $A$: we can effectively enumerate a strong compression of it.

To this end we  use \citep[Lemma 2.1]{iplBarmpaliasHLM13} in the following form:
\begin{lem}\label{aiOlfcSNci}
Given $c$, a computable $f:\Nat\to\Nat$ and computable enumerations
$(A_s)$, $(D_s)$ of \ce sets $A, D$. Suppose that  for all $n, t$ and  $s<t$:
\[
|(A_t-A_s)\restr_{n}|>c
\hspace{0.3cm}\Rightarrow\hspace{0.3cm}
|D_t\restr_{f(n)}|>|D_s\restr_{f(n)}|.
\] 
Then $C(A\restr_n\ \mid D\restr_{f(n)})=\bigon$.
\end{lem}
\begin{thmx}[Theorem \ref{nolyEgDBIB}]
Given any computable enumeration of any set we can effectively enumerate a strong compression of it.
\end{thmx}
\begin{proof}
Given \ce $A$ we show how to effectively enumerate $D$ so that   
\[
C(A\restr_n\ \mid D\restr_{\floor{n/2}})=\bigon.
\]
We may assume that $\abs{A_{s+1}-A_s}\leq 1$ and $A_s\subseteq [0,s]$ for each $s$.

{\em Enumeration.} At each stage $s>0$,  if there is $n\in [3, s]$ such that
\[
16 \ \big | \ |A_s\restr_{2^n}| \fand |A_s\restr_{2^n}|> |A_{s-1}\restr_{2^n}|
\]
we enumerate $\min\left\{[2^{n-3}, 2^{n-2})-D_{s-1}\right\}$ into $D$, for the least such $n$.

{\em Verification.}
For each $n$ we have:
\[
\abs{[2^{n-3}, 2^{n-2})\cap D}\ \leq \ |A\cap[0, 2^n)|/16\ \leq\ 2^{n-4}< \ \abs{[2^{n-3}, 2^{n-2})}.
\]
This shows that  $(D_s)$ is well-defined, in the sense that
$[2^{n-3}, 2^{n-2})-D_{s-1}\neq\emptyset$ for each $n\geq 3$ and  $s>0$. 
By the definition of $(D_s)$: 
\[
|(A_t-A_s)\restr_{n}|>16
\hspace{0.3cm}\Rightarrow\hspace{0.3cm}
|D_t\restr_{\floor{n/2}}|>|D_s\restr_{\floor{n/2}}|.
\] 
 for all $n$, $s<t$. By Lemma \ref{aiOlfcSNci}, 
$C(A\restr_n\ \mid D\restr_{\floor{n/2}})=\bigon$ as required.
\end{proof}

\section{Gainless compression of  enumerations }\label{ufRuaOhnON}
We show how to obtain gainless compressions of enumerations:
\begin{thmx}[Theorem \ref{VXOXdPQyhO}]
Given any computable enumeration of a set $A$ we can effectively enumerate a gainless compression $D\subseteq A$ of it. 
\end{thmx}
As we explain below, the compression obtained by the argument of \S\ref{4r6M2bsUrA}
contains considerable additional information outside the source that we code. 
So gainless compression, in particular a proof of Theorem \ref{VXOXdPQyhO},
require a different method. Since the proof is somewhat involved, 
we first outline of the method and the required gadgets in
 \S\ref{Ojn2jmmXH5}, and use these in the formal construction of  \S\ref{qAtGCX6lh}.

\subsection{Outline for gainless compression }\label{Ojn2jmmXH5}
Let $(A_s)$ be a computable enumeration of the \ce set  $A$.
By Lemma \ref{aiOlfcSNci} it suffices to define a computable enumeration $(D_s)$ of $D$
with $\abs{D\restr_{2n}}\leq\abs{A\restr_{2n}}/2$ and 
\begin{align}
\abs{(A_s-A_t)\restr_{n}}\geq 8\ &\Rightarrow\ \abs{(D_s-D_t)\restr_n} >0\label{OLlwOMFWbU}\\
\abs{(D_s-D_t)\restr_{n}}>1\ &\Rightarrow\ \abs{(A_s-A_t)\restr_{n}} >0.\label{zZJrgEfj2g}
\end{align}
We  assume that all enumerated sets are infinite, at most one number is enumerated at any stage, and no number is enumerated at stage 0.
The first $t$ stages of  $(A_s)$ can be represented by a vector of length $t$, whose $s$th coordinate is,
a number $n$, if $n\in A_{s}-A_{s-1}$, and a {\em dot} $\cdot$ if no enumeration occurred at $s$. 

For example, consider enumerations
\begin{equation}\label{jI9Gdd3lY}
(\cdot,3,\cdot,5,\cdot,\cdot, 0,\cdot,\cdot)
\hspace{0.3cm}\textrm{and}\hspace{0.3cm}
(\cdot,\cdot,1,\cdot,\cdot,5,\cdot,\cdot, 3)
\end{equation}
The first one indicates the enumerations:  three at stage 2, five at stage 4, zero at stage 7, and no enumerations occurred in the remaining stages $t\leq 9$.
These vectors are depicted by a table in Figure \ref{MbJYQoyb8G}, by gray and black respectively. For simplicity, in the case of a joint  
enumeration of two sets $A=(A_s), D=(D_s)$,  we
assume that at any stage, at most one of the sets performs an enumeration. 
The columns of the table correspond to the stages of the enumeration,
while the rows correspond to the numbers that may be enumerated.
When a number $n$ enters $A$ at stage $s$, we color the the cells with coordinates $(i,s), i\geq n$ gray;
a similar action is done for $D$, but with the color black.

We call this the {\em $(A,D)$-table}, which will be very handy in 
visualizing the timing relationships and constraints between the two enumerations in 
the construction of the required compression $D$ of $A$.
We  identify columns and rows in the table by their index numbers, in expressions such as `the largest column' with a given property.
\begin{figure}
\begin{center}
\small 
\rotatebox{90}{\hspace{-0.5cm}$\longleftarrow$ \parbox{2mm}{\em numbers\vspace{0.00cm}} }
\renewcommand{\arraystretch}{1.1}
 \arrayrulecolor{black!90}\setlength{\tabcolsep}{6pt}
\begin{tabular}{|c|c|c|c|c|c|c|c|c|c|c|} 
\multicolumn{11}{c}{\em \hspace{0.3cm}  stages $\longrightarrow$ \hspace{4cm} \vspace{0.1cm}} \\ \hline
0 &1  &2  &3  & 4 &5  &6  & 7 &8  & 9 &$\cdots$ \\ \hline\hline
0 &  &  &  &  &  &  &\cellcolor{gray!30}    &  &  & \\ \hline
1 &  &  & \cellcolor{black!70} &  &  &  &  \cellcolor{gray!30}&    &  & \\ \hline
2 &  &  & \cellcolor{black!70} &  &  &  &  \cellcolor{gray!30}&    &  & \\ \hline
3&  &  \cellcolor{gray!30}& \cellcolor{black!70} &  &  &  &  \cellcolor{gray!30}&  &  \cellcolor{black!70}&   \\ \hline
4 &  &\cellcolor{gray!30}  &  \cellcolor{black!70}&  &  &  &  \cellcolor{gray!30}&  & \cellcolor{black!70} &   \\ \hline
5 &  &\cellcolor{gray!30}  & \cellcolor{black!70} &\cellcolor{gray!30}  &  &\cellcolor{black!70}  &  \cellcolor{gray!30}&  &  \cellcolor{black!70}&   \\ \hline
6 &  &\cellcolor{gray!30}  & \cellcolor{black!70} & \cellcolor{gray!30} &  & \cellcolor{black!70} &  \cellcolor{gray!30}&  &  \cellcolor{black!70}&   \\ \hline
{\tiny $\vdots$} &  &\cellcolor{gray!30}  & \cellcolor{black!70} & \cellcolor{gray!30} &  &\cellcolor{black!70}  & \cellcolor{gray!30}&  &\cellcolor{black!70}  &   \\ \hline
\end{tabular}
\end{center}
\caption{Table depicting the enumerations in \eqref{jI9Gdd3lY}.}\label{MbJYQoyb8G}
\end{figure}
Let  $t_s(n)$ be the largest column with a black-bar starting from row $\leq n$: 
\[
t_s(n):=\max\sqbrad{t < s}{D_{t}\restr_{n+1}\neq D_{t-1}\restr_{n+1}}.
\]
Each row $r$ consists of {\em cells}, which we call the {\em $r$-cells}, and may be divided by 
several black-bars.  The $r$-cells that lie strictly between such black-bars are called {\em $r$-blocks} or {\em blocks in row $r$}. 
Once an $r$-block is formed, the number of gray cells in it is the {\em load of the block}.

Since at stage $s$ a gray or black bar can only appear on column $s$,  
\begin{equation}\label{HEBPcRiJsU}
\textrm{the load of an $r$-block remains constant}
\end{equation}
as it does not change from the stage where the block is formed.
The {\em tail-block} of row $r$ at stage $s$ contains the cells strictly between $t_s(n)$ and $s$. Formally, the
\begin{itemize}
\item  {\em $r$-block}, or block $(b,b')$ in row $r$,  is the interval of $r$-cells  strictly between columns $b,b'$, which are consecutive values of $t_s(r), s\in \Nat$
\item  {\em tail-block} of row $r$ at $s$ is $T_s(n):=(t_s(n),s)$
\item  {\em load} of a block or tail-block $(b,b')$ at row $r$ is $(A_{b'}-A_{b})\restr_{r+1}$.
\end{itemize}
Condition \eqref{OLlwOMFWbU} requires that the load in the blocks of each row is bounded by $8$. 

Define the number of $A$-enumerations $\leq n$ since the last $D$-enumeration $\leq n$:
\begin{equation}\label{4GxJSmzCL}
a_n(s):=\absb{(A_s-A_{t_n(s)})\restr_{n+1}}
\hspace{0.3cm}\textrm{(load of the tail-block on row $n$ at $s$).}
\end{equation}
If $(A_s-A_{s-1})\restr_{n+1}\neq\emptyset$ and the tail-block of row $n$ becomes a block at the next stage,
$a_n(s)$ becomes the load of the new $n$-block.
So, to keep the block-loads  $\leq 8$, it suffices to keep the loads of the tail-blocks similarly bounded:
\begin{equation}\label{CJaVPayayl}
\forall n\leq s,\ a_n(s)\leq 8
\ \Rightarrow\ \textrm{\eqref{OLlwOMFWbU} holds.}
\end{equation}
The following are straightforward from the above definitions:
\begin{itemize}
\item $t_n(s)\leq t_{n+1}(s)\wedga t_n(s)\leq t_{n}(s+1)$, hence $T_{n+1}(s)\subseteq T_n(s)$ 
\item tail-block $T_n(s)$ at $s+1$  becomes  $(t_s(n),s+1)$ or  $(s+1,s+1)=\emptyset$
\item $T_n(s)\neq T_{n+1}(s)\iff n+1\in D_s-D_{t_n(s)}$
\end{itemize}
and $a_n(s+1)\leq a_n(s)+1$.
\begin{defi}
We say that row $r$ is {\em $p$-loaded}
at $s$ if $a_n(s)\geq p$. 
\end{defi}
Ensuring that half of the blocks of each row are sufficiently loaded will give $\abs{D\restr_{n}}\leq\abs{A\restr_{n}}/2$.
The {\em $D$-stages} are when $D$-enumerations occur; the remaining  are called {\em $A$-stages}. 
By slowing down the enumeration $(A_s)$ we ensure that $A,D$ get to enumerate all their members at disjoint sets of stages.
The enumeration at certain $A$-stages $s$ prompts a request for a $D$-enumeration, making $s+1$ a $D$-stage,
also making stage $s+2$ available for a new $A$-enumeration.

{\em Informal construction}.
A straightforward definition of $D$ would be to look for stages where 
a 8-loaded tail-block appears at some row $r$, and eliminate them by enumerating $r$ in $D$. However, in certain cases this greedy strategy enumerates into $D$ almost every number enumerated into $A$. The solution is to occasionally make smaller enumerations than necessary, hence emptying the formation of certain future
heavily loaded tail-blocks in advance. 

\subsection{Gainless compression: construction and verification}\label{qAtGCX6lh}
The idea for construction is that whenever a row becomes 8-loaded at $A$-stage $s$, we determine a {\em target} interval $[n,m]$ such that row $m$ is 8-loaded, row $n$ is 4-loaded and row $n-1$ is not 4-loaded. We enumerate $n$  into $D$ at $s+1$, which becomes a $D$-stage, and by the end of stage $s+1$, all 8-loaded rows have been eliminated.
A $D$-enumeration  will require sufficiently loaded {\em nearby} blocks; this feature
limits the gain in the compression of $A$ into $D$, giving \eqref{zZJrgEfj2g}. 

{\bf Construction.} At $s>0$, if there is a least 8-loaded row $m$:
\begin{itemize}
\item let $n\leq m$ be the least such that for each $t\in [n,m]$ the $t$-row is 4-loaded
\item make  $s+1$ a $D$-stage, say that $[n,m]$ is the {\em target} and enumerate $n$ into $D$. 
\end{itemize}
Otherwise, go to the next stage.

{\bf Verification.}
We  show that the construction is well-defined: when $n$ is requested to be enumerated in $D$ at $s+1$,
we have $n\not\in D_s$.

\begin{lem}\label{GNuDCpKdWL}
The following hold:
\begin{enumerate}[\hspace{0.3cm}(a)]
\item if row $n$ is the least 4-loaded at $A$-stage $s$ then $n\in A_s-A_{t_s(n)}$.
\item each $n$ is enumerated in $D$ at most once.
\end{enumerate}
In particular, $D\subseteq A$.
\end{lem}\begin{proof}
By the hypothesis, $a_{s}(n-1)<a_s(n)$, and by   
\eqref{4GxJSmzCL} of the tail-load of  rows $n, n-1$, 
this can only happen if  $n\in A_s-A_{t_s(n)}$. Hence (a) holds.

If $n$ is enumerated into $D$ at  $s+1$,
row $n$ is the least 4-loaded at $s$ so 
\[
n\in A_s-A_{t_s(n)}
\wedga t_{s+1}(n)=s+1 \wedga n<s+1
\]
by (a).
The above condition can never hold for larger $s$, therefore $n$ cannot be enumerated into $D$ a second time.
So (b) holds.
\end{proof}

\begin{lem}\label{nJpZCPukRH}
There is no 8-loaded row in the $(A,D)$-table.
Hence \eqref{OLlwOMFWbU} holds.
\end{lem}\begin{proof}
At any stage $s$, at the start of which
a $k$-loaded row appears in the $(A,D)$-table, it is removed by the end of $s$ through
an enumeration into $D$. By Lemma  \ref{GNuDCpKdWL} all requested $D$-enumerations are possible,
hence at the end of $s$ there is no 8-loaded row in the $(A,D)$-table.
\end{proof}

\begin{lem}\label{JqLGiUFnn6}
Given a row $r$ and stage $s$, suppose  there are $m_1$ many 4-loaded $r$-blocks and the remaining $m_0$ are not 4-loaded.
Then $m_1\geq m_0$.
\end{lem}\begin{proof}
We label each block row by row by type-1, type-2 and type-3. We then show that type-1 and type-2 blocks are at least $4$-loaded, and in each row there are at least as many type-1 blocks as type-3 blocks. The lemma then follows. 

In row $0$ there is no block to label. Suppose that we have labeled blocks in row $r$ and we move to blocks in row $r+1$. Most blocks in row $r+1$ are just inherited from blocks in row $r$ where we keep their label. The only exception is that $r$ is enumerated into $D$ at some stage $s$. It either \begin{enumerate}[(i)]
	\item generates a new block in the end of row $r$, in which case we label it as type-1; 
	\item divides a row $r$ block $B$ of type-1 into two blocks $B_1$ and $B_2$ in row $r+1$, where $B_1$ is on the left of $B_2$, in which case we label $B_1$ as type-1 and $B_2$ as type-2; 
	\item divides a row $r$ block $B$ of type-2 or type-3 into two blocks $B_1$ and $B_2$ in row $r+1$, where $B_1$ is on the left of $B_2$, in which case we label $B_1$ as type-1 and $B_2$ as type-3. 
\end{enumerate}

	To prove that there are at least as many type-1 blocks as type-3 blocks, we do an induction on $r$ and note that the number of $(\text{type-1},\text{type-2},\text{type-3})$ block changes in the above cases are $(1,0,0)$ for case (i), $(0,1,0)$ for case (ii) and $(1,-1,1)$ or $(1,0,0)$ for case (iii). None of these breaks the fact that there are at least as many type-1 blocks as type-3 blocks. 

	By the choice of the target interval in the construction, each type-1 or type-2 block is exactly $4$-loaded when formed. Also any block that is inherited from a at least $4$-loaded block is also at least $4$-loaded. Therefore all type-1 or type-2 blocks are $4$-loaded. This completes the proof of the Lemma. 
\end{proof}

\begin{lem}
$\abs{D\restr_{n}}\leq \abs{A\restr_{n}}/2$ and \ $\abs{(D_s-D_t)\restr_{n}}>1\ \Rightarrow\ \abs{(A_s-A_t)\restr_{n}} >0$.
\end{lem}\begin{proof}
Toward the first clause, note that each $D$-enumeration $\leq n$ produces a block in row $n$. By Lemma \ref{JqLGiUFnn6}, at least
half of these $\abs{D\restr_{n}}$ many blocks are 4-loaded, so $\abs{A\restr_{2n}}\geq 4\abs{D\restr_{n}}/2$. Hence $\abs{D\restr_{2n}}\leq \abs{A\restr_{n}}/2$.

Toward the second clause, suppose that there is no enumeration into $A\restr_n$ in the interval of stages $(t,s]$. 
If there is no $D\restr_n$-enumeration in $(t,s]$, there is nothing to prove.
Otherwise, 
some $n'<n$ is enumerated into $D$ at a stage $s'\in (t,s]$, so row $r$ for any $n'\leq r\leq n$ is $0$-loaded at stage $s'+1$, and row $r$ for any $r<n'$ is at most $7$-loaded. 

Since $(A_s-A_t)\restr_{n}=\emptyset$, the load of row $r$ for all $r\leq n$ does not increase until stage $s$. By construction, any later $D$-enumeration $<n$ requires a target interval $[d', d]$ with $d'<n$. 

If $d<n'$, then row $d$ is $8$ loaded by construction, but also at most $7$-loaded by the above discussion, a contradiction. If $d\geq n'$, then row $\max\{d,n\}$ is at least $4$ loaded by construction, but also $0$ loaded by the above discussion, also a contradiction. Therefore, during the stages in $(t,s]$ there can be at most one enumeration in $D\restr_n$.
\end{proof}

This completes the proof of Theorem \ref{VXOXdPQyhO}.

\section{Density of enumerations}\label{5OMy18z6Fq}
The question of density in the Kolmogorov complexity of \ce sets
has remained open, despite the extensive work cited in \S\ref{ZUXHugFEb} as well as: 
\begin{enumerate}[(i)]
\item the \ce sets are {\em not dense} in $\leq_{ibT}, \leq_{cl}$\ \  \citep{BLibT,apal/Day10}.
\item the \ce sets are downward dense in $\leq_{rK}, \leq_C, \leq_K$\ \  \citep{iplBarmpaliasHLM13}.
\end{enumerate}
where $\leq_{ibT}$,  $\leq_{cl}$ from Table \ref{ABuEidRZu} 
are stronger versions of $\leq_{rK}$ introduced by \citet{MR2030512}.
The non-density of stringent versions of $\leq_{rK}$ and the downward density of it indicates the non-triviality of the problem.
The  method of \citet{Sacks64} for the interpolation of $B$ between $A,C$ involves simultaneous coding
of $A$ and  parts of $C$ into $B$. This double-coding becomes  challenging in preorders which map $n$-bit segments of one set
to $n$-bit segments of another set. 

We show that the compressibility of enumerations that we introduced 
and studied is crucial in adapting the Sacks density method to $\leq_{rK}, \leq_C, \leq_K$.
\begin{thmx}[Theorem \ref{ba9Ccy8LTPa}]
Let $A, B$ be \ce sets and $r\in \{rK,K,C\}$. If $A, B$ are well-compressible 
and $B<_{r} A$, there is a  \ce set $F$ with $B<_{r} F<_{r} A$.
\end{thmx}
Toward the proof of Theorem \ref{ba9Ccy8LTPa} we need:
\begin{lem}\label{o4KWPEixhN}
Given a \ce set $A$ and $k>0$, 
the following are equivalent:
\begin{enumerate}[\hthree (i)]
\item $A$ is well-compressible  
\item there exists \ce $D$ such that $A\equiv_{rK} D\oplus \emptyset$. 
\end{enumerate}
\end{lem}\begin{proof}
Since $D$ is a gainless strong compression of $D\oplus \emptyset$ we get (ii)$\to$(i).

Assuming (i),  there exists \ce $D$ such that 
\[
C(A\restr_n\ \mid D\restr_{\floor{n/2}})=\bigon
\wedga
C(D\restr_{\floor{n/2}}\ \mid A\restr_n)=\bigon.
\]
Since 
\[
C(D\oplus \emptyset\restr_n\ \mid D\restr_{\floor{n/2}})=\bigon
\wedga
C(D\restr_{\floor{n/2}}\ \mid D\oplus \emptyset\restr_n)=\bigon
\]
we get $A\equiv_{rK} D\oplus \emptyset$ so (ii) holds.
\end{proof}
By Lemma \ref{o4KWPEixhN} it is not hard to see that any pair of well-compressible sets has a least upper bound
with respect to $\leq_{rK}$. We do not know if this is true of all c.e.\ sets, and this question is related to the density question for 
$\leq_{rK}$. It is also non-trivial since some pairs of \ce sets do not even have a common upper bound with respect to 
the stronger $\leq_{ibT}, \leq_{cl}$ in the \ce sets \cite{barmp05ciesw, fanyunear}.

A similar fact holds for $\epsilon$-compression, as defined in \S\ref{VjB4aXSE24}.
Let 
\[
A\oplus_k B:=\sqbrad{k\cdot i}{i\in A}\cup \sqbrad{k\cdot j+t}{j\in B\wedga t\in (0, k)}
\] 
so $\oplus=\oplus_2$. Given $k>0$, by a similar argument,
$A$ has a gainless strong $1/k$-compression iff $A\equiv_{rK} D\oplus_k\emptyset$ for some \ce set $D$.

The proofs in \S\ref{hnMTnBLF4b}, \S\ref{iRG4I4er9e} are based on 
the  density method of \citet{Sacks64}, a priority argument.
Although the presentation is self-contained, the reader would benefit from
prior knowledge of this standard technique from computability theory.

\subsection{Density of enumerations for $rK$}\label{hnMTnBLF4b}
Suppose that $A, B$ are well-compressible \ce sets with $B<_{rK} A$, and
let $\leq_{ibT}$ denote Turing reducibility with oracle-use the identity function.
By Lemma \ref{o4KWPEixhN} there exist \ce $B\asto, A\asto$ such that
$B\equiv_{rK}  B\asto \oplus \emptyset$, $A\equiv_{rK}  \emptyset \oplus A\asto$.
The construction will monitor the numbers entering $A\asto$ and direct some of them into $D$. So 
$D\subseteq A\asto$, $D\leq_{ibT} A$ and $B\leq_{rK} B\asto\oplus D\leq_{rK} A$. 
For $B<_{rK} B\asto\oplus D <_{rK} A$ we also need:
\[
B\asto\oplus D\not\leq_{rK} B \wedga A\not\leq_{rK} B\asto\oplus D.
\]
\begin{defi}
An {\em $rK$-functional} is a \ce operator $\Phi$ on $\twomel$ with
\begin{itemize}
\item $\Phi(\sigma)$ is a \ce subset of $2^{|\sigma|}$, uniformly  in $\sigma$
\item $\abs{\Phi(\sigma)}=\bigon$ and every string in $\Phi(\tau)$ has a prefix in $\Phi(\sigma)$
\end{itemize}
for each $\sigma\prec\tau$.
The extension of $\Phi$ to $\twome$ is: 
\[
\Phi(A):=\sqbrad{X\in\twome}{\forall n,\ X\restr_n\in \Phi(A\restr_n)}.
\]
Let $(\Phi_i)$ an effective enumeration of all $rK$ functionals.
\end{defi}

We define the {\em lengths of agreement}:
\begin{align*}
p_s(e):=&\max\sqbrad{\ell}{\exists t\leq s\  \parb{B\asto\oplus D\restr_{\ell}\ \in \Phi_e(B\restr_{\ell})}[t]}\\[0.2cm]
q_s(e):=&\max\sqbrad{\ell}{\exists t\leq s\  \parb{A\restr_{\ell}\ \in \Phi_e(B\asto\oplus D\restr_{\ell})}[t]}
\end{align*}
and let $p(e):=\lim_s p_s(e)$, $q(e):=\lim_s q_s(e)$.

Since the $\Phi_e$ are $rK$ reducibilities and  $p_s(e), q_s(e)$ are nondecreasing in $s$: 
\begin{align}
B\asto\oplus D\not\in \Phi_e(B)&\iff  p(e)<\infty\label{2hLcQvny5M}\\[0.2cm]
A\not\in \Phi_e(B\asto\oplus D)&\iff q(e)<\infty.\label{2hLcQvny5Ma}
\end{align}
We will satisfy the priority list $P_0>N_0>P_1>N_1>\cdots $ of requirements:
\[
P_e:\ B\asto\oplus D\not\in \Phi_e(B)
\hspace{0.3cm}\textrm{and}\hspace{0.3cm}
N_e:\ A\not\in \Phi_e(B\asto\oplus D)
\]
Let $(A\asto_s)$ be a computable enumeration of $A\asto$ such that $\abs{A\asto_{s+1}-A\asto_{s}}= 1$

We define $D$ by filtering the enumerations $a\in A\asto_{s+1}-A\asto_{s}$ under the rules:
\begin{enumerate}[\hspace{0.3cm}(i)]
\item if $a< p_s(e)$ then $P_e$ wishes to enumerate $a$ into $D$ at stage $s+1$
\item if $a< q_s(e)$ then $N_e$ wishes to avoid enumerating $a$ into $D$ at stage $s+1$
\end{enumerate}
prioritized according to the list of the requirements.
At stage $s+1$ we say that {\em $P_e$ requires attention}  if (i) holds, and say that
{\em $N_e$ requires attention} if (ii) holds.

{\bf Enumeration of $D$.} At stage $s+1$, if $a\in A\asto_{s+1}-A\asto_{s}$
let $e$ be the least such that $P_e$ or $N_e$ requires attention, if such exists, and:
\[
\textrm{if $P_e$ requires attention and $a>\max_{i<e} q_s(i)$,  enumerate $a$ into $D$.}
\]
Otherwise, go to the next stage.

{\bf Verification.}
Clearly, $D\subseteq A\asto$, so  as we explained above, $B\leq_{rK} B\asto\oplus D\leq_{rK} A$. 

Assuming $A\not\leq_{rK} B$, by \eqref{2hLcQvny5M} , \eqref{2hLcQvny5Ma} it remains to show: 
$\forall i,\ \parb{p(i)<\infty\wedga q(i)<\infty}$.

We use induction on $i$: suppose that the claim holds for all $i<e$, 
so $p(i), q(i), i<e$ exist and there exists 
$k_e$ be such that $k_e>p(i), q(i), i<e$.

For a contradiction, assume that $p(e)=\infty$ or $q(e)=\infty$. 

If $p(e)=\infty$ then $B\asto\oplus D\leq_{rK} B$ by \eqref{2hLcQvny5M}.
Since $p_s(e)$ is non-decreasing in $s$ and $D\subseteq A\asto$, we get 
$\emptyset\oplus A\asto \leq_{rK} B\asto\oplus D$, so $A\leq_{rK} B$, a contradiction.

If $q(e)=\infty$, by \eqref{2hLcQvny5Ma} 
we get $A\leq_{rK} B\asto\oplus D$. 
Since $q_s(e)$ is non-decreasing in $s$, 
we get that $D$ is computable, so
$B\asto\oplus D\leq_{rK} B\asto\oplus \emptyset\leq_{rK} B$, and then 
$A\leq_{rK} B$. This contradicts the hypothesis,  concluding the proof of $q(e)<\infty$, the induction step and the proof of 
the clause of Theorem \ref{ba9Ccy8LTPa} for $rK$.

\subsection{Density of enumerations for $K,C$}\label{iRG4I4er9e}
We adapt f \S\ref{hnMTnBLF4b} to $K,C$, establishing the remaining parts of 
Theorem \ref{ba9Ccy8LTPa}.
\begin{lem}\label{gRbSp5sUg}
If $r\in\{C, K\}$ and  $A, D, E$ are \ce sets with  
\[
D\oplus\emptyset\leq_r A
\hspace{0.3cm}\textrm{and}\hspace{0.3cm}
\emptyset\oplus E\leq_r A 
\]
then $D\oplus E\leq_r A$.
\end{lem}\begin{proof}
Let $(D_s), (E_s)$ be computable enumerations of $D,E$, and without loss assume  that $0\in D$.
We state the argument for $r=C$, as the other case is similar.
Let $m_n$ be the the number that is the last enumeration in $(D\oplus E)\restr_n$.

Let $\eqa$ denote equality
up to a universal constant and similarly for $\geqa$.
Then 
\[
C(A\restr_n)\ \geqa\ \max\big\{C((D\oplus\emptyset)\restr_n), C((\emptyset\oplus E)\restr_n)\big\}\ \eqa\ 
C(m_n)
\]
 and  $C(m_n)\ \eqa\  C((D\oplus E)\restr_n)$, so $D\oplus E\leq_C A$ as required.
\end{proof}
We write the proof for $\leq_K$, as the case for $\leq_C$ is similar. 

Let $A,B$ be \ce with  $B<_K A$ and as in \S\ref{hnMTnBLF4b}, let $B\asto, A\asto$ be \ce such that
\[
B\equiv_{rK}  B\asto \oplus \emptyset
\hspace{0.3cm}\textrm{and}\hspace{0.3cm}
 A\equiv_{rK}  \emptyset \oplus A\asto
\]
which exist by Lemma \ref{o4KWPEixhN}.
Again, we direct some of the $A\asto$-enumerations into $D$, in real time, so $D\leq_{ibT} A$ and  
$B\asto\oplus D\leq_{K} A$ by Lemma \ref{gRbSp5sUg}. 

Since $B \leq_{rK}  B\asto\oplus D$, we get $B \leq_{K}  B\asto\oplus D$. 
We also need  that
\[
A\not\leq_K  B\asto\oplus D
\hspace{0.3cm}\textrm{and}\hspace{0.3cm}
B\asto\oplus D\not\leq_K B\asto
\]
so we diagonalize against the  indexed relations:
\[
X\leq^e_K Y\overset{{\rm def}}{\iff} \forall n\ \ K(X\restr_{n})\leq K(Y\restr_{n})+e.
\]
To ensure that the following prioritized requirements 
\[
P_e:\ B\asto\oplus D \not \leq^e_K B
\hspace{0.3cm}\textrm{and}\hspace{0.3cm}
N_e:\ A\not \leq^e_K B\asto\oplus D
\]
are met, we redefine parameters
$p_s(e), q_s(e)$ of \S\ref{hnMTnBLF4b} with respect $\leq_K$:
\begin{align*}
p_s(e):=&\max\sqbrad{\ell}{\exists t\leq s\ \ K_t(B_t\asto\oplus D_t\restr_{\ell})\ \leq K_t(B_t\restr_{\ell})+e}\\[0.2cm]
q_s(e):=&\max\sqbrad{\ell}{\exists t\leq s\ \  K_t(A_t\restr_{\ell})\ \leq K_t(B_t\asto\oplus D_t\restr_{\ell})+e}
\end{align*}
which are nondecreasing in $s$, and set $p(e):=\lim_s p_s(e)$, $q(e):=\lim_s q_s(e)$.

Since $K_t(A_t\restr_{\ell}), K_t(B_t\restr_{\ell}), K_t(B_t\asto\oplus D_t\restr_{\ell})$ reach a limit as $t\to\infty$: 
\begin{align}
B\asto\oplus D\leq^e_K B\iff  & p(e)=\infty\label{CXTSlXa29g} \\[0.3cm]
A \leq^e_K B\asto\oplus D\iff &  q(e)=\infty\label{CXTSlXa29ga}
\end{align}
which are the analogues of \eqref{2hLcQvny5M} and \eqref{2hLcQvny5Ma}.

The enumeration of $D$ is identical to the construction of \S\ref{hnMTnBLF4b},
but with respect to the modified definitions of $p_s(e), q_s(e)$.
The verification is also similar.  

{\bf Verification.}
Clearly, $D\subseteq A\asto, D\leq_{ibT} A\asto$ so  as explained above: 
\[
B\leq_{rK} B\asto\oplus D\leq_{K} A. 
\]
Assuming $A\not\leq_{K} B$, by \eqref{CXTSlXa29g} it remains to show: 
$\forall i,\ \parb{p(i)<\infty\wedga q(i)<\infty}$.

We use induction on $i$: suppose that the claim holds for all $i<e$, 
so $p(i), q(i), i<e$ exist and there exists 
$k_e$ be such that $k_e>p(i), q(i), i<e$.

Toward $p(e)<\infty$, assume for a contradiction that
$p(e)=\infty$. By \eqref{CXTSlXa29g} we have $B\asto\oplus D\leq_{K} B$.
Since $p_s(e)$ is non-decreasing in $s$ and $D\subseteq A\asto$, we also have
$\emptyset\oplus A\asto \leq_{rK} B\asto\oplus D\asto$, so $A\leq_{K} B$. This contradicts
the hypothesis.
Toward $q(e)<\infty$, assume for a contradiction that
$q(e)=\infty$, so by \eqref{CXTSlXa29ga} 
we get $A\leq_{K} B\asto\oplus D$. 
Since $q_s(e)$ is non-decreasing in $s$, 
we get that $D$ is computable, so
$B\asto\oplus D\leq_{rK} B\asto\oplus \emptyset\leq_{rK} B$, and then 
$A\leq_{rK} B$. This contradicts the hypothesis,  concluding the induction step and the proof of 
the clause of Theorem \ref{ba9Ccy8LTPa} for $K$.

\section{Compression games}\label{pdP8xPFfiU}
The  unanswered question, whether every \ce set $A$ has a gainless strong compression, can be approached in terms of
two-player games between player-1 enumerating $A$ and player-2 attempting to enumerate a compression $D$ of $A$.

\subsection{Balance game}\label{l94XsTPQep}
Two players pick numbers and enumerate numbers in their corresponding sets $A, D$ during the stages of the game.
Let $A_s, D_s$ contain the numbers enumerated in $A,D$ by the end of stage $s$.
At stage $s$: 
\begin{itemize}
\item if $s\in 2\Nat$ player-1  picks any number outside $A_{s-1}$
\item if $s\in 2\Nat+1$   player-2 can either {\em pass}  or pick some $n\in 2\Nat-D_{s-1}$
\end{itemize}
where by {\em pass} we mean that he does not pick any number.\footnote{Alternatively, without loss of generality, 
we could replace condition $D\subseteq 2\Nat$ by the requirement that $D$ does not contain adjacent numbers.}
For each $n,s$ let
\[
t_s(n):=\min\sqbrad{i\leq s}{\parb{(A_s\cup D_s)-(A_i\cup D_i)}\cap [0,n] =\emptyset}
\]
and define the {\em score} $d_s(n)$ of $n$ at $s$ by
\[
d_s(n):=
\begin{cases}
\abs{(A_s-A_{t_s(n)})\cap [0,n]} &\textrm{if $t_s(n)\in 2\Nat$}\\[0.3cm]
\abs{(D_s-D_{t_s(n)})\cap [0,n]}&\textrm{otherwise.}
\end{cases}
\]
The winning condition for player 1 depends on a fixed parameter $k$: 
\[
\exists n,s:\ d_s(n)>k.
\]
Otherwise player 2 wins. This concludes the description of the {\em $k$-balance game}.

Intuitively, the goal of player 2 is to maintain a {\em balance} 
between the $A$ and $D$ enumerations in all initial segments $[0,n]$ up to an error $k$, while being confined to the even numbers. 
By Lemma \ref{aiOlfcSNci} we get:

\begin{prop}\label{VBotw9mjNK}
If player 2 has a computable winning strategy in the $k$-balance game for some $k$ then 
every \ce set is well-compressible.
\end{prop}
A converse of Proposition \ref{VBotw9mjNK} is also true: a winning strategy for player 1 
in the $k$-balance game which is uniformly computable in $k$ can be used (as a `black box') in a priority construction of a c.e.\ set $A$
such that for each $e,k$:
\[
\exists n\ \parb{C(A\restr_n\mid W_e\oplus\emptyset\restr_n)>k\ \vee\ 
C(W_e\oplus\emptyset\restr_n \mid A\restr_n)>k}
\]
where $(W_e)$ is the universal enumeration of all c.e.\ sets. In other words, we get a computable enumeration
of a c.e.\ set which is not well-compressible. 

\subsection{A simpler game}
Solving the game of \S\ref{l94XsTPQep} is essentially equivalent to deciding the density of
$\leq_{rK}$ in the \ce sets, but is non-trivial to do. Toward this goal we study a simpler game
which captures the essence of the balance game.

{\bf $k$-even game.} Assuming that $k\geq 2$, in each round:  
\begin{itemize}
\item player-1 picks a set $R$ of $k$ numbers that he has not picked earlier
\item player-2 picks $n\in 2\Nat\cap [\min R, \max R]$ that he has not picked earlier.
\end{itemize}
If player 2 runs out of moves, he loses the game. Otherwise he wins.
\begin{table}
\colorbox{black!5}{\arrayrulecolor{white!20!black} 
\begin{tabular}{cl}\toprule
{\small $\XX$}  & { a number chosen by both players}\\[0.5ex]
{\small $\OO$}  & { a number chosen by player-1 but not player-2}\\[0.5ex]
{\small $\TT$}  & { a number chosen by player-1 before player-2  responds}\\[0.5ex]
\bottomrule
\end{tabular}}\centering
\caption{Configuration vocabulary  for winning  strategy of player 1.}\label{ABuEidRZud}
\end{table}

It is not hard to see that player-1 has a winning strategy for $k=2$.
Player-1 also wins for $k=3$ but the winning strategy is not trivial.

We first show that a dynamic analysis of the game is necessary:
a winning strategy for player-1 cannot be independent of the moves of player-2.
Indeed, player-1 cannot win the {\em static game} where he reveals all his moves at once, 
even for $k=2$.

\begin{prop}
If $(R_i)_{i<n}$ are disjoint sets of positive integers of size $2$, there are distinct 
even numbers $(d_i)_{i<n}$ such that $\min R_i\leq d_i\leq\max R_i$ for all $i$. 
\end{prop}\begin{proof}
We use induction on $n$. For $n=1$ the claim clearly holds, so  assume that it holds for all $m\leq n$, we prove it for $n+1$.
Without loss of generality we  assume that $1\in R_n$ or $2\in R_n$, since otherwise we can reorder the sequence and repeatedly 
take $R_i' :=R_i-2$. 

First, suppose that both $1,2\in R_n$ or 
\[
\textrm{one of $1, 2$ is in $R_n$, and the other is not in any of $R_i, i<n$.}
\]
In  both cases we have $\min R_i>2$ for all $i<n$. We apply the induction hypothesis on $(R_i)_{i<n}$ to get $(d_i)_{i<n}$. 
Since $d_i\geq\min R_i>2$ for all $i<n$, the sequence $(d_i)_{i<n}$ does not include $2$ so we can set $d_n=2$. 

The remaining case is that  one of $1$ or $2$ is in $R_n$ and the other is in $R_j$ for some $j<n$. 
Let $a_n$ be such that $R_n-\{1,2\}=\{a_n\}$ and $a_j$ be such that $R_j-\{1,2\}=\{a_j\}$. Without loss of generality we assume that $a_n<a_j$ and set $R_j'=\{a_n,a_j\}$ and $R_i'=R_i$ for all $i<n$, $i\neq j$. Now $\min R_i'>2$ for all $i<n$. 
We apply the induction hypothesis on $(R_i')_{i<n}$ to get $(d_i)_{i<n}$. 

Again $2$ does not appear in $(d_i)_{i<n}$ so we can set $d_n=2$. Then
\begin{itemize}
\item $\min (R_j)\leq 2$ so $\min (R_j)\leq d_j$
\item $d_j\leq\max (R_j')=\max\{a_n,a_j\}=a_j$
\end{itemize}
so $d_j\leq\max (R_j)$ and $(d_i)_{i\leq n}$ is a solution. 
\end{proof}

Toward constructing a winning strategy for player-1 in the 3-even game, we
use symbols $\XX,\OO, \TT$  as shown in Table \ref{ABuEidRZud}, and strings of them to denote the  configurations 
realized during the game. 
Formally, a \textit{configuration} is a finite string of these letters, each corresponds to a number, and they are ordered such that letters on the left correspond to smaller numbers. 

For the strings of $\XX,\OO, \TT$ we assume that
\begin{itemize}
\item no number is chosen by either player between the number represented by any two adjacent letters
\item there is sufficient space around the number represented by each letter: 
 even numbers before and after the numbers that are not chosen by either player 
\end{itemize}
where the exact space required is easily computed in each case. 

For instance, when player-1 has enumerated $2,4,100,200,302,500$ and player-2 has enumerated $100,302$, we may say that a configuration $\XX\OO\XX$ is produced, where the three letters represent $100,200$ and $302$. 

We also apply subscripts to $\XX, \OO, \TT$ when we need to distinguish them.

\begin{figure}
\scalebox{0.75}{\begin{tikzpicture}[
nodeA/.style={rectangle,  minimum height=9mm,  thick, draw=gray, fill=gray!20, outer sep=6pt, inner sep=6pt},
nodeB/.style={ellipse,   minimum height=13mm,  outer sep=4pt, inner sep=2pt,   thick, draw=gray, fill=gray!5},  
nodeC/.style={ellipse,  outer sep=4pt, inner sep=6pt,   thick, draw=gray, fill=gray!5},  
]
 \node (XX) [nodeA] at (-5,2) {$\XX\XX$};
 \node (evenforce) [nodeB] at (3,2) {Player 2 has to produce adjacent even numbers};
 \node (adjEven) [nodeA] at (3,0) {adjacent even numbers};
 \node (p2loses) [nodeC] at (-3,0) {Player 2 loses};
 \draw [->] (XX) to (evenforce);
 \draw [->, dashed] (evenforce) to (adjEven);
 \draw [->] (adjEven) to (p2loses);
 \end{tikzpicture}}
\centering
\caption{Final part of the strategy of player 1  in the 3-even game.}\label{evenwinAb}
\end{figure}

\begin{thm}\label{A64lgqSrt7}
Player-1 has a winning strategy in the 3-even game.
\end{thm}\begin{proof}
We construct a winning strategy for player-1 where he always chooses even numbers except for the last step. With this feature, note that:
\begin{enumerate}[(i)]
\item if player-2  has chosen $2n$ which is not yet chosen by player-1, the latter can win by choosing $R=\{2n-1,2n,2n+1\}$
\item if player-2 has chosen $2n$ and $2n+2$ then player-1 can win by choosing $R=\{2n-1,2n+1,2n+3\}$
\item  any first move by player-1 (with sufficient space in-between) followed by a response by player-2 yields configuration $\XX$.
\end{enumerate} 
By (i)  we may assume that player-2 only chooses even numbers that have been chosen by player-1.
By (ii), (iii) it remains to construct a  strategy  such that: 
\begin{enumerate}[\hthree(a)]
\item from $\XX$ forces the game into configuration $\XX\XX$ (Figure \ref{evenwinA})
\item from  $\XX\XX$ force player-2 to choose  adjacent even numbers (Figure \ref{evenwinAb}).
\end{enumerate}
Part (b) is straightforward: assuming that a stage is reached where $n$, $n+k$ have been previously chosen by both players, if  
player-1 picks $R=\{n-2, n+2, n+k+2\}$
then player-2 can only pick a number in $R$. The latter results in player-2 having chosen
two adjacent even numbers. 

Finally we show how to achieve (a), as illustrated in Figure \ref{evenwinA}.
From $\XX$:
\begin{itemize}
\item player-1 picks $\TT_1$, $\TT_2$, $\TT_3$  yielding $\TT_1$ $\XX$ $\TT_2$ $\TT_3$
\item then player-2 can only choose one of  $\TT_1, \TT_2, \TT_3$. 
\end{itemize}
If player-2 chooses $\TT_1$ or $\TT_2$, configuration $\XX\XX$ as required. 
\begin{figure}\scalebox{0.7}{\begin{tikzpicture}[
nodeA/.style={rectangle,  minimum height=6mm,  thick, draw=gray, fill=gray!20, outer sep=6pt, inner sep=6pt, font=\small },
nodeB/.style={ellipse,  outer sep=4pt, inner sep=6pt,   thick, draw=gray, fill=gray!5, font=\small}, 
nodeC/.style={rectangle,  outer sep=4pt, inner sep=6pt,  minimum height=6mm, minimum width=18mm,  thick, draw=gray, fill=gray!20, font=\small },
nodeD/.style={cloud, cloud puffs=18,cloud ignores aspect,  outer sep=4pt, inner sep=6pt,   thick, draw=gray}, ]
 \node (X) [nodeA] at (3,2) {$\XX$};
 \node (TXTT) [nodeB] at (5.5,2) {$\TT\XX\TT\TT$};
 \node (TXTX) [nodeA] at (5,-0.5) {$\TT\XX\TT\XX$ which is $\XX\OO\XX$};
  \node (TXOTXT) [nodeB] at (9.5,-0.5) {$\TT\XX\OO\TT\XX\TT$};
 \node (XXTT) [nodeC] at (13,3) {$\XX\XX\TT\TT$};
 \node (TXXT) [nodeC] at (13,2) {$\TT\XX\XX\TT$};
 \node (XXOTXT) [nodeC] at (13,1) {$\XX\XX\OO\TT\XX\TT$}; 
 \node (TXXTXT) [nodeC] at (13,0) {$\TT\XX\XX\TT\XX\TT$}; 
 \node (TXOXXT) [nodeC] at (13,-1) {$\TT\XX\OO\XX\XX\TT$}; 
 \node (TXOTXX) [nodeC] at (13,-2) {$\TT\XX\OO\TT\XX\XX$};  
 \node (allXX) [nodeD] at (17,0.5) {all are $\XX\XX$};  
\draw [->] (X) to (TXTT);
\draw [->,  out=-100, in = 80, dashed] (TXTT) to (TXTX);
\draw [->, dashed] (TXTT) to (TXXT);
\draw [->, out=10, in = 180, dashed] (TXTT) to (XXTT);
\draw [->, out=45, in = 180, dashed] (TXOTXT) to (XXOTXT);
\draw [->, out=15, in = 180, dashed] (TXOTXT) to (TXXTXT);
\draw [->, out=-11, in = 180, dashed] (TXOTXT) to (TXOXXT);
\draw [->, out=-35, in = 180, dashed] (TXOTXT) to (TXOTXX);
\draw [->, out=0, in = 180] (TXTX) to (TXOTXT);
\draw [->, out=0, in = 180, dotted] (TXOTXX) to (allXX);
\draw [->, out=0, in = 180, dotted] (TXOXXT) to (allXX);
\draw [->, out=0, in = 180, dotted] (TXXTXT) to (allXX);
\draw [->, out=0, in = 180, dotted] (XXOTXT) to (allXX);
\draw [->, out=0, in = 180, dotted] (XXTT) to (allXX);
\draw [->, dotted, out=0, in = 180, dotted] (TXXT) to (allXX);
\end{tikzpicture}}
\centering
\caption{Main part of the strategy of player 1.}\label{evenwinA}
\end{figure}
Otherwise he chooses $\TT_3$, resulting in  $\XX\OO\XX$, which we write as $\XXss_1\OOss_1\XXss_2$. Then
\begin{itemize}
\item player-1 picks $\TT_1,\TTss_2,\TTss_3$  yielding $\TTss_1\XXss_1\OOss_1\TTss_2\XXss_2\TTss_3$
\item player-2 can only choose one of  $\TT_1, \OO_1,\TT_2, \TT_3$. 
\end{itemize}
If player-2  chooses any of $\TT_1,\OO_1,\TT_2$, $\TT_3$ then  $\XX\XX$ is produced, as required.
\end{proof}

This proof does not extend to $k>3$, in which case
we do not know the winner.

\section{Conclusion and  problems}\label{dzforlz6pr}
A natural notion of compression and gain of enumerations was introduced, 
and the importance of gainless compression in 
the Kolmogorov complexity of computably enumerable sets was demonstrated. 
We showed that every \ce set is compressible certain ways while our methods fall short of answering the following questions:
\begin{enumerate}[\hspace{0.3cm}(i)]
\item is every \ce set well-compressible?
\item are the \ce sets dense in the $\leq_K$, $\leq_C$, $\leq_{rK}$ degrees?
\end{enumerate}
We showed that a positive answer to (i) gives a positive answer to (ii) while both are related to the existence of
a least upper bound for any pair of c.e.\ sets in these preorders. 
Indeed, well-compressible \ce sets have a least upper bound in the \ce sets
and the density of $\leq_K$ in the \ce sets would follow from 
a positive answer to
\[
\textrm{does every \ce set $A$ have a \ce set $D$ with $A\equiv_K D\oplus\emptyset$ ?}
\]
We reduced (i) to the solution of a two-player enumeration game, which we approached by solving a simplified version of it.
The outcome of these games depends on the order of enumerations, so they  are outside the framework of
\cite{LachlanGames, KummerGamesTAMS}.

\bibliographystyle{abbrvnat}
\bibliography{sinum}
\end{document}